\documentclass[12pt]{colt2014}

\title[Le Cam meets LeCun]{Le Cam meets LeCun: Deficiency and Generic Feature Learning}

\coltauthor{\Name{Brendan {van Rooyen}} \Email{brendan.vanrooyen@anu.edu.au}\\
\Name{Robert C. Williamson} \Email{bob.williamson@anu.edu.au}\\
\addr Australian National University and NICTA, Canberra ACT 0200, Australia}

\usepackage{times}
\usepackage{bbold,amsmath,bm}
\usepackage{float}
\usepackage{graphicx}
\usepackage{grffile}
\usepackage[all]{xy}
\usepackage{booktabs}
\usepackage{times}

\newcommand{\RR}{\protect\mathbb{R}}
 
\newcommand{\EE}{\protect\mathbb{E}}

\newcommand{\dist}{\raise.17ex\hbox{$\scriptstyle\mathtt{\sim}$}}
\newcommand{\Lbar}{\underline{L}}

\DeclareMathOperator*{\arginf}{arg\,inf}

\newtheorem{mydef}{Definition}
\newtheorem{myex}{Example}
\newtheorem{mythm}{Theorem}

\newtheorem{mylemma}{Lemma}

\begin{document}

\maketitle
\thispagestyle{empty}

\begin{abstract}
``Deep Learning" methods attempt to learn generic features in an unsupervised fashion from a large unlabelled data set. These generic features should perform as well as the best hand crafted features for any learning problem that makes use of this data. We provide a definition of generic features, characterize when it is possible to learn them and provide algorithms closely related to the deep belief network and autoencoders of deep learning. In order to do so we use the notion of deficiency distance and illustrate its value in studying certain general learning problems.
\end{abstract}

\section{Introduction}

``Deep" unsupervised feature learning methods \citep{Bengio2009,Hinton2006,Vincent2008,Tenenbaum2000,Lecun2013} present a challenge to learning theory. This paper takes up this challenge, of explaining when and why these techniques work. Is it possible to learn generic features from data in an unsupervised fashion that perform well in a multitude of tasks, and if so how do we learn these features? Following the work of the statistician Lucien Le Cam \citep{Cam2011,Le1964,Lecam1974} and utilizing the techniques of statistical decision theory, in particular the comparison of statistical experiments \citep{Blackwell1951,Blackwell1953,Torgersen1991,Ferguson1967} we show (theorem 3) that it is possible to construct generic features $Z$ from data $X$ if and only if one can find a encoder/decoder pair 
$$
\xymatrix{
X \ar[rr]^{\text{encoder}} & & Z \ar[rr]^{\text{decoder}} & & X
}
$$
with low probability of reconstruction error. Furthermore, the worst case difference in performance of the best decision rule that uses such features versus the best decision rule that uses the raw data is bounded above by the probability of reconstruction error. We also show that we can learn this encoder/ decoder pairing in a hierarchical fashion and that the probability of reconstruction error of such a``stacked" system is bounded by the sum of the probability of reconstruction errors of each layer.

While our approach is abstract, the ultimate pay off will be a novel inequality (theorem 2) that provides a characterization of when generic feature learning is possible. This inequality coupled with the concept of deficiency \citep{Cam2011,Torgersen1991} (to be explained in the paper) illuminates the algorithms used in deep learning and provides means to judge the generic quality of the features learnt by such methods.

\section{The General Learning Problem}

For all of the following assume that all of the measure spaces $\Theta,A,X,Y,Z$ and so on are finite. This does not restrict any of the results, rather it allows for a cleaner presentation free of measure theoretic technicalities as well as boundedness and existence concerns.

A learning problem is a quintuple $(\Theta,X,T,A,L)$. $\Theta$ is a set of possible ``true hypothesis" or unknowns. While we cannot observe $\Theta$ directly, we can observe data in some set $X$. $T$ is a relationship between the two sets $\Theta$ and $X$ called the \emph{experiment}. $T(\theta)$ tells us what data we expect to see if $\theta$ is the true hypotheses. Ultimately we are required to make a decision by choosing an action $a\in A$, and our performance is measured by a loss function $L: \Theta \times A \rightarrow \RR$. We view the loss as an integral part of a learning problem and as such do not place any restrictions on it other than boundedness. As is usual in statistical learning, for our possible relationships we use markov kernels (conditional probability assignment/stochastic matrices).
\begin{mydef}

A \emph{Markov kernel} $T: \Theta \rightsquigarrow X$ is a function from $\Theta$ to $\mathcal{P}(X)$, the set of probability distributions on $X$.

\end{mydef}
For arbitrary sets $Y$ and $Z$, denote by $M(Y,Z)$, the set of all Markov kernels from $Y$ to $Z$. As we can represent $\mathcal{P}(Y)$ by vectors in $\RR^{|Y|}$ with positive entries we have $\mathcal{P}(Y) \subset \RR^{|Y|}$. As such one can represent a Markov kernel $T: Y \rightsquigarrow Z$ as an $|Z| \times |Y|$ matrix of positive entries where the sum of all entries in each column is equal to $1$. It is easily verified that $M(Y,Z)$ is a closed convex subset of $\RR^{|Z|\times |Y|}$, the set of all $|Z| \times |Y|$ matrices. 

A function $f: Y\rightarrow Z$ defines a Markov kernel $F$ with $F(y) = \delta_{f(y)}$, a point mass distribution on $f(y)$. For every measure space $X$, there are two special Markov kernels, the identity (or completely informative) Markov kernel from the identity function $\operatorname{id}_X: X \rightarrow X$, and the completely uninformative Markov kernel from the function $\bullet_X : X \rightarrow \bullet$ from $X$ to a one element set $\bullet$.

From a prior distribution $\pi\in \mathcal{P}(\Theta)$ and a Markov kernel $T: \Theta \rightsquigarrow X$ we can construct a joint distribution $T\otimes \pi \in \mathcal{P}(\Theta \times X)$. Using the matrix vector representation, this is achieved by post multiplying $T$ with a diagonal matrix with $\pi$ on the diagonal, $T\otimes \pi = T \operatorname{diag}(\pi)$. This is no different to the standard product rule $P(\theta,x) = P(x|\theta) P(\theta)$.

Given a prior distribution $\pi\in \mathcal{P}(\Theta)$ and a Markov kernel $T: \Theta \rightsquigarrow X$ we denote the Markov kernel obtained by Bayes rule by $T^*: X \rightsquigarrow \Theta$. 

A learning problem can more compactly be represented as the pair $(L,T)$ where $\Theta,A,X$ can be inferred from the type signatures of $L$ and $T$. We measure the size of loss functions by $\lVert L \rVert_\infty = \sup_{\theta,a} |L(\theta,a)|$

\subsection{Decision Rules}

Upon observing data $x\in X$ we are required to relate $x$ to a set of actions $A$ by some other Markov kernel $d:X \rightsquigarrow A$ known as a (randomized) decision rule.
{\large $$
\xymatrix{
\Theta \ar@{~>}[rr]^T && X \ar@{~>}[rr]^d && A}
$$}
We are judged on the quality of the composed relation $d\circ T : \Theta \rightsquigarrow A$.

\textbf{Definition (Composition)} \emph{Suppose $T_1: X \rightsquigarrow Y$ and $T_2: Y \rightsquigarrow Z$ are Markov kernels. Then we can compose $T_1$ and $T_2$ yielding $T_3: X \rightsquigarrow Z$ by matrix multiplication 
$$
T_3 = T_2 \circ T_1 = T_2 T_1.
$$}
A Markov kernel $T: X \rightsquigarrow Y$ provides a function $T : \mathcal{P}(X) \rightarrow \mathcal{P}(Y)$ by matrix multiplication. To calculate $T(\pi), \pi \in \mathcal{P}(X)$ we identify $\pi$ with a vector and $T$ with a matrix and use matrix multiplication. This function is convex linear.

\subsection{Risk and Value: Ranking Decision Rules and Learning Problems}

Given a learning problem $(\Theta,X,T,A,L)$ one can rank decision rules $d:X \rightsquigarrow A$ using the full Bayes risk
\begin{align*}
R_L : \mathcal{P}(\Theta) \times M(\Theta,A) &\rightarrow \RR \\
(\pi,D) &\mapsto \EE_{\theta \dist \pi}\EE_{a\dist D(\theta)} L(\theta,a). 
\end{align*} 
Here $\pi \in \mathcal{P}(\Theta)$ is a prior distribution on $\Theta$ which reflects which hypotheses we feel are more or less likely to be true. Note that both $\mathcal{P}(\Theta)$ and $M(\Theta,A)$ are convex sets and that $R_L$ is convex bilinear (the same as bilinear but restricted to convex combinations). Alternately, taking a supremum over the prior yields the max risk.

\textbf{Ranking Learning Problems}. We also rank the difficulty of learning problems. The greatest challenge in a learning problem comes from the fact we can not use an arbitrary decision rule $D \in M(\Theta,A)$. Rather, we are restricted to a certain subset of $M(\Theta,A)$ that ``factors through" $T$. We are only allowed to use the data we see.
\begin{mydef}[Factoring Through]	
Suppose we have two Markov kernels $T: X \rightsquigarrow Z$ and $U: X \rightsquigarrow Y$. We say that $U$ \emph{factors through} $T$ (written $T | U$) if there exists a Markov kernel $U/T: Z \rightsquigarrow Y$ such that $U = (U/T) \circ T$. Denote by
$$
M(X,Y)_T := \{U\in M(X,Y): U = (U/T) \circ T \ \text{for some}\ U/T \in M(Z,Y) \}.
$$
\end{mydef}
If $T|U$ then $U$ can be thought of as $T$ with extra noise $U/T$. The reader is directed to section 1 of the appendix for more properties of factoring through. In this notation $d = D/T$. If $T|D$ we have
\begin{align*}
R_L(\pi,D) &= R_L(\pi, (D/T)\circ T) \\
&= \EE_{\theta \dist \pi} \EE_{a \dist (D/T) \circ T(\theta)} L(\theta,a) \\
&= \EE_{\theta \dist \pi} \EE_{x \dist T(\theta)} \EE_{a \dist (D/T)(x)} L(\theta,a).
\end{align*}
We assign a \emph{value}
$$
\mathcal{V}(\pi,\Theta,X,T,A,L) = \mathcal{V}_L(\pi,T) := \inf_{D \in M(\Theta,A)_T} R_L(\pi,D)
$$
to a learning problem, with lower value being better. Taking a supremum of the value over the prior yields the minimax risk. The value is the risk of the best possible decision rule for the learning problem at hand. 

\textbf{Bayes Decision Rules are Optimal}. If we use $V_L(\pi,T)$ to order decision rules then the best $D \in M(\Theta,A)_T$ is found by using Bayes rule
$$
(D/T)(x) = \arginf_a \EE_{\theta \dist T^*(x)} L(\theta,a),
$$
with risk
$$
R_L(\pi,D) = \mathcal{V}_L(\pi,T) = \EE_{x \dist \pi_X} \inf_a \EE_{\theta \dist T^*(x)} L(\theta,a) = \EE_{x \dist \pi_X} \Lbar(T^*(x)), \\
$$
where $\Lbar : \mathcal{P}(\Theta) \rightarrow \RR$, $\Lbar(\pi) = \inf_a \EE_{\theta \dist \pi} L(\theta,a)$. Hence $\Lbar$ is concave. This results allows us to parametrize the action set $A$ by $\mathcal{P}(\Theta)$, by taking 
\begin{align*}
f: \mathcal{P}(\Theta) &\rightarrow A \\
\pi &\mapsto \arginf_a \EE_{\theta \dist P} L(\theta,a)
\end{align*}
effectively properising the loss function. $\hat{L}(\theta,\pi) := L(\theta,f(\pi))$ is a proper loss \citep{Reid2009b,Dawid2007,Grunwald2004,Parry2012}. There are deep connections between $\Lbar$ and $\hat{L}$, we review some of these in section two of the appendix. 

\textbf{Connections to other Information Measures.}

There are many connections between $V_L(\pi,T)$ and different information measures present in the literature.

\textbf{Definition} \emph{For a convex $f: \RR_+^{n-1} \rightarrow \RR$ the $f$-information of a set of $n$ distributions $P_1,\dots, P_{n}\in \mathcal{P}(X)$ is 
$$
I_{f}(P_1,\dots, P_{n}) := \int_X f(\frac{d P_2}{d P_1}, \dots, \frac{d P_{n}}{d P_1}) d P_1.
$$}
$f$-informations are a multi distribution extension of $f$ divergences and are used in certain generalizations of rate-distortion theory where they produce better bounds than the standard techniques \citep{Ziv1973,2573956,Reid2009b,Garca-Garca}. For suitable choices of $f$ one can recover more known measures of information such as mutual information.

\textbf{Theorem} \emph{For all experiments $T$, loss functions $L$ and priors $\pi$, the gap between then the value of $T$ and the least informative experiment $\bullet_\Theta$ is a $f$-information for suitable $f$
$$
\mathcal{V}_L(\pi,\bullet_\Theta) - \mathcal{V}_L(\pi,T) = I_{f}(T(\theta_1),\dots, T(\theta_n)) = I_{f}(T)
$$}
We direct the reader to \citet{Reid2009b,Garca-Garca}. By the bijections presented in these two papers, one can replace $V_L(\pi,T)$ by these divergences in all that follows. In particular any $|\mathcal{V}_L(\pi,T) - \mathcal{V}_L(\pi,U)|$ can be replaced with a $|I_{f}(T) - I_{f}(U)|$ for suitable $f$, with no effect on the result. The reader is directed to section 3 of the appendix for proof.

\subsection{Examples}

Here we present some examples of familiar learning problems phrased in this more abstract language. Normally there is a distinction between \emph{learning algorithms}, something that takes a data set of $n$ instance-label pairs and produces a classifier, and a \emph{decision rule} that is the learnt classifier. Here we do not make such a distinction. Both learning algorithms and decision rules produce \emph{actions}, hence we only use the term decision rule.

\begin{myex}[Classification]

$\Theta = \{-1,1\}$ and a Markov kernel $T: \Theta \rightsquigarrow X$ is then a pair of distributions $T(1) = P$, $T(-1)=Q$ on $X$. Normally $A = \Theta$ and a decision rule picks the corresponding label for a given observed $x \in X$. Different losses could be used, eg the 0-1 loss $L_{01}$.

\end{myex}
\begin{myex}[Supervised Learning]

There is a space of labels $Y = \{-1,1\}$ and a space of covariates $Z$ with $X = (Z \times Y)^n$. $\Theta = \mathcal{P}(Z \times Y)$ the space of joint distributions on $Z \times Y$ with $T: \Theta \rightsquigarrow X$ the map that sends each distribution to its n-fold product. $A$ is then some set of classifiers (eg linear hyperplanes/kernel machines and so on). A decision rule then produces a classifier $a$ from n pairs $(z_i,y_i)$. For example empirical risk minimization algorithms, $ERM: X \rightarrow A$, pick the classifier that minimizes the empirical loss on the observed training set. Many suitable losses exists but normally $L(\theta,a)$ is the misclassification probability of the classifier $a$ when used against the distribution $\theta$.

\end{myex}
\begin{myex}[Generalized Supervised Learning]
$\Theta$ and $A$ are the same as supervised learning, although the data observed is different. For example in semi-supervised learning we have $X = (Z \times Y)^n \times Z^m$, we observe $n$ instance label pairs and $m$ instances. $T:\Theta \rightsquigarrow X$ then maps each joint distribution $\theta$ to a product of $n$ copies of itself and $m$ copies of its marginal distribution over instances.
\end{myex}
\begin{myex}[Active Learning]

$\Theta$ could be anything with $X = Z^n$, length n sequences in some set $Z$. Each active learning policy determines a different $T: \Theta \rightsquigarrow X$.

\end{myex}

\section{Feature Learning}

Starting from a learning problem $(\Theta,X,T,A,L)$, feature learning methods aim to extract features $\phi: X \rightsquigarrow Z$. One then bases all decisions on these features.
{\large $$
\xymatrix{
\Theta \ar@{~>}[rr]^T && X \ar@{~>}[rr]^\phi && Z \ar@{~>}[rr]^d && A}
$$}
These methods swap the original learning problem with $(\Theta,Z,\phi \circ T,A,L)$. Normally the space $Z$ is smaller/ of lower dimension than $X$ and aims at presenting a ``compressed" view of the information contained in $X$. Features can be used for several reasons including communication/ storage constraints, increased performance (ie by implementing decision rules based on Z rather than X directly), knowledge discovery and to avoid ``curse of dimensionality" problems.

\subsection{Supervised Feature Learning}

There has been much attention in the Machine Learning literature on supervised feature learning techniques, were $\Theta,T,L$ and the prior $\pi$ are fixed. These methods construct features by minimizing the \emph{feature gap}
$$
\Delta\mathcal{V}_L(\pi,\phi,T) := \mathcal{V}_L(\pi,\phi \circ T) - \mathcal{V}_L(\pi,T).
$$ 
There is now a general framework for solving such problems based largely on variations of the Blahut-Arimoto Algorithm from Rate Distortion theory \citep{Banerjee2005,Tishby2000,Cover2012}. For particular choices of $T$ and $L$ these methods reproduce many clustering methods such as k-means. We review these methods in section 4 of the appendix. These feature learning methods are not general enough for our purposes as they rely on both the experiment and loss. For example if we wish to \emph{learn} $T$ from data given by $(\theta, x)$ pairs, then we are required to learn the features after we have learnt the experiment $T$. Ideally we would like to learn a feature map $\phi: X \rightsquigarrow Z$ independently from $T$ so that learning $\phi \circ T$ is just as beneficial as learning $T$ no matter what $T$ is.

\subsection{Generic Feature Learning}

For many learning problems, a large amount of unlabelled/loosely labelled data $X$ is readily available. For example, with any problem involving images one only has to enter some basic search queries into google to be presented with millions of instances. One of the main arguments of the deep learning community is that while this data may not be of direct use in learning classifiers, it can be of great use in learning feature representations. There exists many methods in the literature to learn features from unlabelled data.

In line with these methods we consider the following relaxation of the supervised feature learning problem. We assume that $\Theta, T,A,L$ and the prior $\pi$ are allowed to vary leaving only $X$ fixed, with one restriction. We assume that there is enough unlabelled data collected from the marginal distribution $\pi_X$ on $X$ that we are able to form an accurate estimate of $\pi_X$. We consider all learning problems and priors $\pi$ that are consistent with this information about $X$ ie with $T(\pi) = \pi_X$. We then seek to find features $\phi$ so that the value of $\mathcal{V}_L(\pi,\phi \circ T)$ is as close to $\mathcal{V}_L(\pi,T)$ as possible, no matter what what $\pi, \Theta, T,A,L$ are. To ensure that minor differences in value are not exploited by multiplying the loss function by a large constant, we penalize the size of the loss function by $\lVert L \rVert_\infty$.

\begin{mydef}[Generic Features]

Fix a measure space $X$ and a distribution $\pi_X$. $\phi: X \rightsquigarrow Z$ are \emph{generic features of quality $\epsilon$} for $X$ if for all learning problems $(\Theta,X,T,A,L)$ and priors $\pi$ with $T(\pi) = \pi_X$ we have
$$
\Delta\mathcal{V}_L(\pi,\phi,T) \leq \epsilon \lVert L \rVert_{\infty}.
$$

\end{mydef} 
Ideally we want to make $\epsilon$ as small as possible, and if $\epsilon$ is $0$ then our features do not ever decrease the value. 

For our more relaxed problem the value of our features is effectively 
\begin{align*}
&\sup_{L: \lVert L \rVert_\infty \leq 1} \sup_T \sup_{\pi \ : \ T(\pi) = \pi_X} \Delta\mathcal{V}_L(\pi,\phi,T) \\
&= \sup_{\Theta} \sup_A \sup_{L\in \RR^{\Theta \times A}: \lVert L \rVert_\infty \leq 1} \sup_T \sup_{\pi \ : \ T(\pi) = \pi_X}\Delta\mathcal{V}(\Theta,A,X,\phi,T,\pi)
\end{align*}
Luckily supremums like these have been tackled in theoretical statistics particularly in the work of Lucien Le Cam \citep{Cam2011,Le1964,Lecam1974}. In his 1964 paper Le Cam coined the deficiency distance as an extension of David Blackwell's ordering of experiments \cite{Blackwell1953,Blackwell1951} and as a means to provide an approximate version of the statistical notion of sufficiency. This quantity was used in his later work to form a metric not just on probability distributions but on experiments, and in particular allows one to calculate supremums over all loss functions and priors with fixed $\Theta,T$. We introduce this  quantity (the deficiency) in the next section.

\section{Approximate Factoring Through and Deficiency}

Suppose $T: \Theta \rightsquigarrow X$ and $U: \Theta \rightsquigarrow Y$ are Markov kernels where $U$ does not factor through $T$. We measure the degree to which $U$ fails to factor through $T$ by the \emph{weighted directed deficiency} \citep{Torgersen1991}.
$$
\delta_{\pi}(T,U) := \inf_{V: X \rightsquigarrow Y} \EE_{\theta \dist \pi} \lVert U(\theta) - V \circ T(\theta) \rVert.
$$
$\lVert P - Q \rVert$ is the variational divergence between the distributions $P,Q \in \mathcal{P}(X)$, a standard metric on probability distributions (see section 5 of the appendix for properties). Calculating weighted directed deficiencies is a convex (actually linear) optimization problem. One has 
\begin{align*}
f(\pi,V) &= \EE_{\theta \dist \pi} \lVert U(\theta) - V \circ T(\theta) \rVert \\
&= \EE_{\theta\dist \pi} \lVert U(\theta) - V \circ T(\theta) \rVert_1 \\
&= \lVert U \otimes \pi - (V \circ T) \otimes \pi \rVert_1 \\
&= \lVert U \ \operatorname{diag}(\pi) - V\ T\ \operatorname{diag}(\pi)  \rVert_1
\end{align*}
is linear in $\pi$. Since the variational divergence $||P - Q||$ is convex in $Q$, $f$ is also convex in $V$, because $\lVert U \ \operatorname{diag}(\pi) - V\ T\ \operatorname{diag}(\pi)  \rVert_1$ is the composition of a linear function and a convex function. Hence determining weighted directed deficiencies is a $l_1$ minimization problem. Fast methods exist for solving this problem (eg the well known simplex method of linear programming). 

Taking a supremum over the prior $\pi$ yields the \emph{directed deficiency}, 
$$
\delta(T,U) := \sup_\pi \delta_{\pi}(T,U)  =\inf_{V: X \rightsquigarrow Y} \sup_{\theta} \lVert U(\theta) - V \circ T(\theta) \rVert.
$$
where the second follows from the minimax theorem \citep{Komiya1988}. For the sake of checking whether $T|U$ it suffices to use the weighted directed deficiency and a prior that does not put zero probability on any $\theta$. In this case $\delta_{\pi}(T,U) = 0$ if and only if $T|U$ \citep{Torgersen1991}. The \emph{weighted deficiency}, and \emph{deficiency} are respectively
\begin{align*}
\Delta_\pi(T,U)&:=\max(\delta_\pi(T,U),\delta_\pi(U,T)) \\
\Delta(T,U) &:= \sup_\pi \Delta_\pi(U,T).
\end{align*}
$\Delta(T,U)=0$ if and only if $T|U$ and $U|T$, when $T$ is \emph{isomorphic} to $U$ written $T \cong U$. The deficiency distance is a true metric on the space of experiments (modulo isomorphic experiments). A proof of this is included in the appendix.

\subsection{Relation to Risk}

Factoring through and approximate factoring through are deeply related to the worst case difference in performance between two learning problems with the same $\Theta$ as the loss is varied. Here we state the three theorems that highlight the connections between factoring through and risk \citep{Torgersen1991}. Fix $\Theta$ and two experiments $T: \Theta \rightsquigarrow X$ and $U: \Theta \rightsquigarrow Y$. 

\textbf{Theorem (Information Processing)} \emph{If $T|U$ then for any loss function $L$ and prior $\pi$ \\$\mathcal{V}_L(\pi,T) \leq \mathcal{V}_L(\pi,U)$. }

In particular the information processing theorem implies that $\Delta\mathcal{V}_L(\pi,\phi,T)\geq0$.

\textbf{Theorem (Blackwell-Sherman-Stein)} \emph{$T|U$ if and only if $\mathcal{V}_L(\pi,T) \leq \mathcal{V}_L(\pi,U)$ for all loss functions $L$ and priors $\pi$.}

\textbf{Theorem (Randomization)} \emph{Fix $\epsilon>0$, $T,U$ and $\pi$. $\mathcal{V}_L(\pi,T) \leq \mathcal{V}_L(\pi,U) + \epsilon \lVert L \rVert_{\infty}$ if and only if $\delta_\pi(T,U) \leq \epsilon$ for all loss functions $L$.}

These three theorems allow one to move between decision theoretic notions such as risk and value to probability theoretic notions such as factoring through. For example the original definition of sufficiency  can be interpreted in terms of factoring through.

\textbf{Theorem} \emph{Fix an experiment $T:\Theta \rightsquigarrow X$ and a function $f: X \rightarrow Y$. f is a \emph{sufficient statistic} if $f \circ T \cong T$.}

By the Blackwell-Sherman-Stein theorem we have an equivalent condition for sufficiency in terms of value.

\textbf{Theorem} \emph{Fix an experiment $T:\Theta \rightsquigarrow X$ and a function $f: X \rightarrow Y$.Then $f$ is a sufficient statistic if $\mathcal{V}_L(\pi,f\circ T) = \mathcal{V}_L(\pi,f\circ T)$ for all $L$ and $\pi$.}

Isomorphic experiments always have the same value, no matter what the loss function or the set of actions. Approximately isomorphic experiments, ones where $\Delta(T,U)$ is small, always have approximately the same value. Due to the similarities between learning features and sufficiency statistics, it should be of no surprise that tools for working with approximate sufficiency appear in feature learning.

The Randomization theorem is an example of an approximate notion in probability theory (here approximate sufficiency) has a dual approximate notion in terms of risk. 
\begin{mythm}

For all experiments $U,T$ and all priors $\pi$
$$
\Delta_\pi(U,T) = \sup_L \frac{|\mathcal{V}_L(\pi,U) - \mathcal{V}_L(\pi,T)|}{\lVert L \rVert_{\infty}}
$$
\end{mythm}

This result is an improvement and generalization of the result contained in \citet{Liese2012} that applied only to binary experiments, $\Theta \cong \{-1,1\}$, and held with inequality. The proof is included in the appendix. We utilize this theorem and the randomisation theorem heavily in the following section.

\subsection{Reductions via Factoring Through}

Approximate factoring through can be used to transform decision rules for one learning problem to rules for another, with a provable bound on the performance of this decision rule. Suppose $D_U\in M(\Theta,A)_U$ is a decision rule used for the learning problem $(\Theta,Y,U,A,L)$. For another learning problem $(\Theta,X,T,A,L)$, we can construct a decision rule from a Markov kernel $V: X \rightsquigarrow Y$
{
\large
$$
\xymatrix{
& \Theta \ar@{~>}[ld]_T \ar@{~>}[rd]^U & &&  \\
X\ar@{~>}[rr]^V && Y \ar@{~>}[rr]^{D_U/U} && A
}
$$
}
$D_T = (D_U/U) \circ V \circ T$. Furthermore if $\epsilon = \EE_{\theta \dist \pi}\lVert U(\theta) - V \circ T(\theta) \rVert$ then 
$$
R_L(\pi,D_T) \leq R_L(\pi, D_U) + \epsilon \lVert L \rVert_{\infty}.
$$
By taking infimums over $V$ we obtain the smallest $\epsilon$ in the above.

\section{Analysis of Generic Feature Learning via Deficiency}

Assume one has enough data in some measure space $X$ to form a good estimate of the marginal distribution $\pi_X$. We wish to construct generic features $\phi: X \rightsquigarrow Z$ for $X$. This is equivalent to finding a $\phi$ that minimizes $\sup_{T}\Delta_\pi(T,\phi \circ T)$. One might imagine that this means finding for each $\phi$ the worst $T$, but this is not the case.
\begin{mythm}

For all experiments $T: \Theta \rightsquigarrow X$, for all measure spaces $Z$ and for all feature maps $\phi: X \rightsquigarrow Z$ 
$$
\Delta_\pi(T,\phi \circ T) \leq \Delta_{T(\pi)}(\operatorname{id}_X,\phi).  
$$

\end{mythm}
No matter which feature map $\phi$ we use, the worst learning problem we can pit against it is the one that asks you to reconstruct $X$ directly from the features. The proof is straightforward and is included in the appendix. It hinges on the representation of $\mathcal{V}_L(\pi,T)$ in terms of average posterior Bayes risk and the Randomization theorem. By definition and as $\operatorname{id}_X | \phi$
\begin{align}
\Delta_{\pi_X}(\operatorname{id}_X,\phi) &= \delta_{\pi_X}(\phi,\operatorname{id}_X) \nonumber \\
&= \inf_{d: Z \rightsquigarrow X} \EE_{x\dist \pi_X}\lVert \operatorname{id}_X(x) - (d\circ\phi)(x)\rVert \nonumber \\
&= \inf_{d: Z \rightsquigarrow X} \EE_{x\dist \pi_X} \lVert \delta_x - (d\circ\phi)(x)\rVert  \\
&= 2 \inf_{d: Z \rightsquigarrow X} \EE_{x\dist \pi_X} \EE_{x' \dist (d\circ\phi)(x) } \mathbb{1}(x' = x), 
\end{align}
where from lines (1) to (2) we have used one of the equivalent forms of the variational divergence listed in the appendix. Hence $\Delta_{\pi_X}(\operatorname{id}_X,\phi)$ is equal to twice the minimal possible average reconstruction error from the encoder $\phi$ and the prior $\pi_X$. This means that finding the best generic features for the data $X$ involves finding the $\phi$  that gives the learning problem $(X,Z,\phi,X,L_{01})$ the highest value $\mathcal{V}_{L_{01}}(\pi_X,\phi)$. We term this problem the \emph{reconstruction problem}.
\begin{mythm}

Fix a prior $\pi_X$ and an $\epsilon >0$.  $\phi$ constitutes generic features of quality $\epsilon$ for $X$ if and only if for the reconstruction problem $(X,Z,\phi,X,L_{01})$, one can find a decision rule $d: Z \rightsquigarrow X$ with 
$$
R_{L_{01}}(\pi_X,d \circ \phi) \leq \frac{\epsilon}{2}.
$$
\end{mythm}
For a proof see appendix. If we optimize over both the encoder $\phi$ and the decoder $d$, finding 
$$
\inf_{\phi, d} \EE_{x\dist \pi_X}\lVert \operatorname{id}_X(x) - (d \circ \phi)(x)\rVert
$$
we obtain a variant of the popular autoencoder algorithm from deep learning \citep{Vincent2008}.

Of course one can always take $\phi= \operatorname{id}_X$ in which case no real feature learning is done and no performance is lost. However, it is more instructive to set $Z$ to a measure space of smaller size/ dimension than $X$ so that the feature learning extracts (provably) useful patterns in $\pi_X$. Performing the joint minimization is a non-convex problem.

\subsection{Relation to other Feature Learning Methods}
\textbf{Infomax}. When faced with a choice of feature maps $\phi_i: X \rightsquigarrow Z_i$, the Infomax principle \citep{Bell1995,Linsker1989} dictates that you should choose the features that minimize the conditional entropy between data and features, $H(X|Z_i)$. By the Hellman-Raviv inequality from information theory \citep{Hellman1970} we have 
$$
\Delta_{\pi_X}(\operatorname{id}_X,\phi_i) \leq H(X|Z_i)
$$
meaning the Infomax principle is minimizing an upper bound of the reconstruction error.

\textbf{Manifold Learning}. Under the assumption that $\pi_X$ has support on some manifold $M \subseteq X$ manifold learning methods aim to extract this manifold and provide a parametrization $\phi: M \rightarrow \RR^n$ \citep{Silva2002,Belkin2003}. If we are able to learn this manifold then any coordinate system would constitute generic features.

\textbf{Sparse Coding}. Much like manifold learning, sparse coding also attempts to find lower dimensional structure in $\pi_X$ \citep{Lee2006,Olshausen1997}. Here $Z$ is chosen to have higher dimension than $X$, however image of the feature map $\phi: X \rightsquigarrow Z$ should comprise only of sparse vectors, those with few non-zero entries. If $\phi$ is injective on the support of $X$ then $\phi$ are generic features.

\subsection{Learning Feature Hierarchies}

One of the tenets of the deep learning paradigm is that features should be learnt in a hierarchical fashion. One should first find patterns in $\pi_X$ through a feature map $\phi$ and then find patterns in $\phi(\pi_X)$ and so on. We construct a chain of feature maps
{\large $$
\xymatrix{
X \ar@{~>}[rr]^{\phi_1} && Z_1 \ar@{~>}[rr]^{\phi_2} && Z_2 \ar@{~>}[rr]^{\phi_3}  && \dots
}
$$}
with final feature space $Z_n$ and final feature map given by the composition of all maps in the chain. Proceeding in this fashion allows greater control over the feature spaces $Z_i$. For example the first could be of similar (but still lower) size than $X$, perhaps before a big drop off in the middle of the chain. If we can learn each of the $\phi_i$ iteratively it also makes searching for features easier. For example, perhaps the first three feature maps have low probability of reconstruction error but not the fourth. In this situation at least we still have a good feature map given by the composition of the first three mappings. We have for any chain of feature maps
\setcounter{equation}{0} 
\begin{align}
&\Delta_{\pi_X}(\operatorname{id}_X, \phi_n \circ \dots \circ \phi_2 \circ \phi_1) \nonumber \\ 
\leq& \Delta_{\pi_X}(\operatorname{id}_X, \phi_1) + \Delta_{\pi_X}(\phi_1, \phi_2 \circ \phi_1) + \dots + \Delta_{\pi_X}(\phi_{n-1} \circ \dots  \circ \phi_1, \phi_n \circ \dots \circ \phi_1) \\
\leq& \Delta_{\pi_X}(\operatorname{id}_X, \phi_1) + \Delta_{\phi_1(\pi_X)}(\operatorname{id}_{Z_1}, \phi_2) + \dots + \Delta_{\phi_{n-1} \circ \dots \circ \phi_2 \circ \phi_1(\pi_X)}(\operatorname{id}_{Z_{n-1}}, \phi_n)
\end{align}
where (1) follows as the weighted deficiency satisfies a triangle inequality and (2) follows from repeated application of theorem 2. The reconstruction error of the entire system is bounded by the sum of reconstruction errors at each step of the chain. This means we can learn a feature mapping iteratively, by first learning patterns in $\pi_X$, then in $\phi(\pi_X)$ and so on. This is exactly the process that occurs in a Deep Belief Network \citep{Hinton2006}

\subsection{Supervised Feature learning can work when Generic Feature Learning Fails}

We present two examples where one can not learn generic features, however we can learn experiment/loss specific features.

\textbf{Experiment Specific Features.} Let $\Theta = \RR$ with $X = \RR^n$ and $T(\theta)$ given by the product of $n$ normal distributions with mean $\theta$ and variance $1$. It is easy to verify that the sample mean $\phi: X \rightarrow \RR$ is a sufficient statistic meaning that at least for this experiment we can greatly compress the information contained in $X$. However, if we take as a prior $\pi$ for $\Theta$ a normal distribution of mean $0$ and variance $1$, then the marginal distribution $\pi_X$ will not be concentrated on a set of smaller dimension nor have any particularly interesting structure. Hence we can not find interesting generic features in this case.

\textbf{Experiment and Loss Specific Features.} Let $\Theta = \{-1,1\}$ with $T(\theta)$ a normal distribution centred on $\theta$ as in the figure below. 

\begin{figure}[h]
\centering
\includegraphics[width=0.8\linewidth]{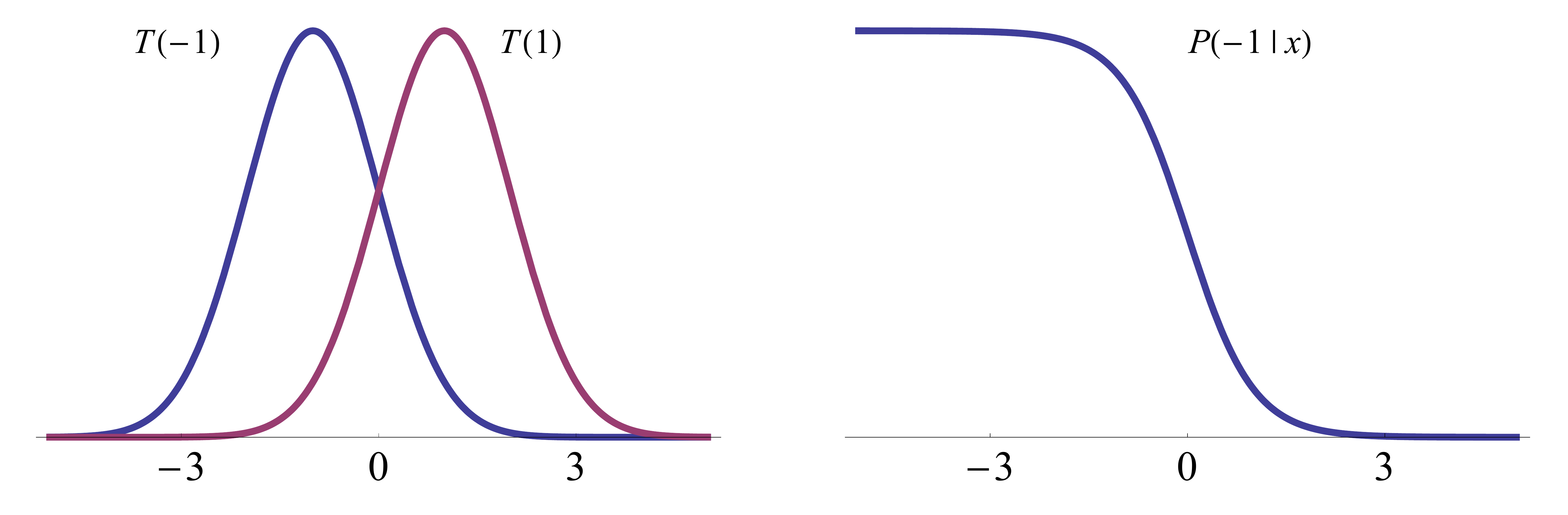}
\caption{Figures for Loss Specific Feature Learning, see text}
\label{fig:Example}
\end{figure}
For this experiment ,$L_{01}$ and a uniform prior $\pi$, the best decision $d(x) = 1$ if $x > 0$ as $P(-1|x)>\frac{1}{2}$ and $d(x) = -1$ otherwise as $P(-1|x)\leq\frac{1}{2}$ . It is easy to show that $\Delta \mathcal{V}_{L_{01}}(\pi,d,T) = 0$, all we need is the output of $d$. However if we change the loss to a cost sensitive loss $L_c$ where say misclassifying a $1$ is more costly than a $-1$, we no longer have $\Delta \mathcal{V}_{L_c}(\pi,d,T) = 0$.

\subsection{Alternate Reconstruction Problems}

The reconstruction problem $(X,Z,\phi,X,L_{01})$ that is required to be solved to construct generic features is the most difficult one we can pose. To perform well in this problem we are required to reconstruct each $x\in X$ \emph{exactly}. This discards other interesting structure the set $X$ may have. For example if $X$ is image data a different loss function $L: X \times X\rightarrow \RR$, perhaps one elicited from psychological tests of what humans perceive to be different images is more appropriate. While these are valid points, we remind the reader that this is a first step in understanding these methods, and making any extra assumptions about $X$ and its structure is exactly what we are trying to avoid. However, the Hellman Raviv inequality does give means of bounding the value $\mathcal{V}_{L_{01}}(\pi,\phi)$ with the value of different reconstruction problems.

\section{Concluding Remarks and Future Work}

We have defined generic features and have provided a characterization of when it is possible to learn them. In doing so we have illuminated some popular feature learning methods including autoencoders, deep belief networks and the Infomax principle.. We have moved from supervised feature learning methods
$$
\inf_{\phi: X \rightsquigarrow Z} \Delta\mathcal{V}_L(\pi,\phi,T)
$$
where $\Theta,A,L,\pi$ are all \emph{fixed} to
\begin{equation}
\inf_{\phi: X \rightsquigarrow Z} \sup_{L: \lVert L \rVert_\infty \leq 1} \sup_T \sup_{\pi \ : \ T(\pi) = \pi_X} \Delta\mathcal{V}_L(\pi,\phi,T) \tag{$\star$}
\end{equation}
with almost \emph{nothing fixed}. Equation $(\star)$ shows how difficult and general finding generic features is. It is reasonable to argue that in practice on does not require features that work for \emph{all} experiments, \emph{all} losses and \emph{all} priors, which by existing results implies a quantification over \emph{all} $f$-informations. This begs the question of \emph{which} experiments, loss functions and priors to consider. We might not require \emph{all} the information in $X$ to be maintained, just enough to suit our purposes. This is analogous to the the problem of formalizing the notion of how much information is contained in an experiment. As argued long ago by Morris Degroot \citep{DeGroot1962}, even if one is doing an experiment to ``gain information", eventually one \emph{does something} with this ``information" by choosing how to act, the consequence of which will be measured by some loss. Hence a more general theory of feature learning needs to be able to control the sensitivity to the loss, allowing one to move from supervised feature learning to generic feature learning.

A starting point would be to take fixed $L$ and $T$ lying in some subset of $M(\Theta,X)$ as occurs in robust statistics \citep{Huber2011}, or to allowing small perturbations in the loss. Deficiency can possibly play a role in the development of algorithms to learn features when we take these restricted supremums. There is scope to develop new quantities and theorems analogous to those for deficiency where instead of a supremum over all losses, one takes a supremum over some restricted subset. At present this is an open and uncharted area of both machine learning and theoretical statistics. 

\newpage

\section{}

\subsection{More Properties of Factoring Through}

\textbf{Lemma} \emph{Factoring through has the following properties. For all sets $W,X,Y,Z$ and Markov kernels $T_1: X\rightsquigarrow Y$ and $T_2: Y \rightsquigarrow W$ we have}

\begin{enumerate}
	\item $M(X,Z)_{T_1} \supseteq M(X,Z)_{T_2 \circ T_1}$
	\item $M(X,Z) = M(X,Z)_{id_X} \supseteq M(X,Z)_{T_1}$
	\item $M(X,Z)_{T_1} \supseteq M(X,Z)_{\bullet_X}$
\end{enumerate}

\begin{proof}
For (1) if $D\in  M(X,Z)_{T_2 \circ T_1}$ then we have 
$$
D = (D/(T_2 \circ T_1)) \circ (T_2 \circ T_1) = ((D/(T_2 \circ T_1)) \circ T_2) \circ T_1 
$$
which is obviously in $M(X,Z)_{T_1}$. For (2) note that for any Markov kernel $D:X \rightsquigarrow Y$ one has $D \circ id_X = D$. Hence all kernels factor through the identity. For (3) take any $D \in M(X,Z)_{\bullet_X}$ and recall that $\bullet_X = \bullet_Y \circ T_1$. Hence 
$$
D = (D/\bullet_X) \circ \bullet_X = (D/\bullet_X) \circ \bullet_Y \circ T_1
$$
which is obviously in $M(X,Z)_{T_1}$
\end{proof}
Note that $M(X,Z)_{\bullet_X}$ comprises the constant Markov kernels, ones that map each $x\in X$ to the same distribution on $Z$. 

One can view factoring through as adding noise. By showing that $T|U$, we are showing that $U$ is $T$ composed with extra noise $U/T$

For some intuition on what factoring through looks like below is a plot of four binary experiments ($\Theta = {-1,1}$). 
\begin{center}
\includegraphics[width=0.8\linewidth]{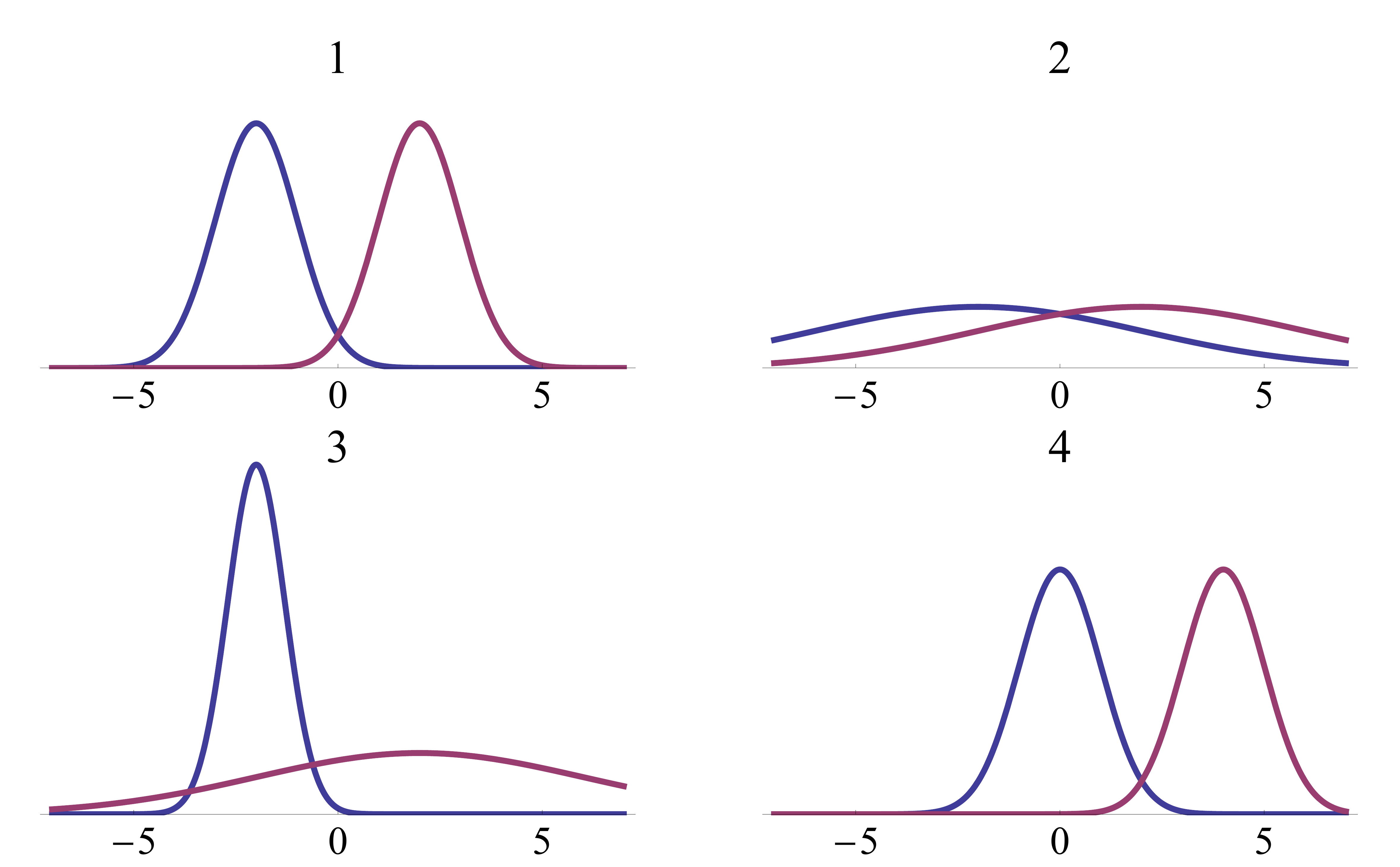}
\end{center}
We have that:
\begin{itemize}

	\item The second factors through the first.
	\item The third does not factor through the first (nor the first through the third).
	\item The fourth factors through the first and vice versa (it is just a shifted version of the first).

\end{itemize}
Suppose $T: \Theta \rightsquigarrow X$ and $U: \Theta \rightsquigarrow Y$ are two Markov kernels. If $T | U$ and $U | T$ then we say that $T$ is isomorphic to $U$ written $T \cong U$. Isomorphic Markov kernels can appear quite different. For example suppose that $T: \Theta \rightsquigarrow X$ is an exponential family distribution ($\Theta$ in this case are the parameters for the family). If $\phi: X \rightarrow \RR^n$ is the sufficient statistics for the family $T$ then $T \cong \phi \circ T$ even though they appear quite different, and that it appears $\phi$ may throw away lots of information. From a statistical point of view isomorphic Markov kernels are the same.


\subsection{$f$-Information and Value}

\textbf{Theorem} \emph{For all experiments $T$, loss functions $L$ and priors $\pi$, the gap between then the value of $T$ and the least informative experiment $\bullet_\Theta$ is a $f$-information for suitable $f$
$$
\mathcal{V}_L(\pi,\bullet_\Theta) - \mathcal{V}_L(\pi,T) = I_{f}(T(\theta_1),\dots, T(\theta_n)) = I_{f}(T)
$$}

We repeat the proof presented in \citet{Garca-Garca}, which is a variant of one presented in \citet{DeGroot1962}.

\begin{proof}
$$
\mathcal{V}_L(\pi,T) = \EE_{x \dist T(\pi)} \Lbar(T^*(x))
$$
and
$$
\mathcal{V}_L(\pi,\bullet_\Theta) = \Lbar(\pi)
$$
where $\Lbar: \mathcal{P}(\Theta) \rightarrow \RR$ is concave. Let $|\Theta| = n$, $\pi_i = \pi(\theta_i)$, $P_i = T(\theta_i)$ and $M = \sum\limits_{i = 1}^{n} \pi_i P_i$. Then $T^*(x) = (\pi_1 \frac{d P_1}{d M}(x), \pi_2 \frac{d P_2}{d M}(x),  \dots, \pi_n \frac{d P_n}{d M}(x))$ and 
\begin{align*}
\mathcal{V}_L(\pi,T) &= \int_X \Lbar(\pi_1 \frac{d P_1}{d M}, \pi_2 \frac{d P_2}{d M},  \dots, \pi_n \frac{d P_n}{d M} ) d M \\
&= \int_X \Lbar(\frac{d P_1}{d M} (\pi_1, \pi_2 \frac{d P_2}{d P_1}, \dots, \pi_n \frac{d P_n}{d P_1}) ) \frac{d M}{d P_1} d P_1 \\
&= \int_X \underbrace{\Lbar(\frac{1}{\pi_1 + \sum\limits_{i = 2}^{n}\pi_i \frac{d P_i}{d P_1}}(\pi_1, \pi_2 \frac{d P_2}{d P_1}, \dots, \pi_n \frac{d P_n}{d P_1}) ) (\pi_1 + \sum\limits_{i = 2}^{n}\pi_i \frac{d P_i}{d P_1})}_{\phi(\frac{d P_2}{d P_1}, \dots , \frac{d P_n}{d P_1})} d P_1
\end{align*}
We need to show that $\phi: \RR_+^{n-1} \rightarrow \RR$ is concave. Note that $\phi = g \circ h$ where 
\begin{align*}
h: \RR_+^{n-1} &\rightarrow \RR_+^{n} \\
v &\mapsto \operatorname{diag}(\pi) (1,v)
\end{align*}
prepends $1$ to $v$ and and multiplies by the prior, and 
\begin{align*}
g: \RR_+^n &\rightarrow \RR \\
v &\mapsto \lVert v \rVert_1 \Lbar(\frac{v}{\lVert v \rVert_1}).
\end{align*}
As $h$ is affine, $\phi$ is concave if $g$ is. For all $\lambda \in [0,1]$ and $v_1, v_2 \in \RR_+^n$ 
\begin{align*}
g(\lambda v_1 + (1- \lambda) v_2) &= \lVert \lambda v_1 + (1- \lambda) v_2 \rVert_1 \Lbar(\frac{\lambda v_1 + (1- \lambda) v_2}{\lVert \lambda v_1 + (1- \lambda) v_2 \rVert_1}) \\
&= \lVert \lambda v_1 + (1- \lambda) v_2 \rVert_1 \Lbar(\frac{\lambda \lVert v_1 \rVert_1 }{\lVert \lambda v_1 + (1- \lambda) v_2 \rVert_1} \frac{v_1}{\lVert v_1 \rVert_1}  +  \frac{(1-\lambda) \lVert v_2 \rVert_1 }{\lVert \lambda v_1 + (1- \lambda) v_2 \rVert_1} \frac{v_2}{\lVert v_2 \rVert_1}  ) \\
&\leq \lambda \lVert v_1 \rVert_1 \Lbar(\frac{v_1}{\lVert v_1 \rVert_1}) + (1 - \lambda)\lVert v_2 \rVert_1 \Lbar(\frac{v_2}{\lVert v_2 \rVert_1}) \\
&= \lambda g(v_1) + (1 - \lambda) g(v_2)
\end{align*}
where we have used the concavity of $\Lbar$. Therefore $\phi$ is concave. Finally 
$$
\mathcal{V}_L(\pi,\bullet_\Theta) - \mathcal{V}_L(\pi,T) = \int_X \Lbar(\pi) - \phi(\frac{d P_2}{d P_1}, \dots , \frac{d P_n}{d P_1}) d P_1 = I_f(\frac{d P_2}{d P_1}, \dots , \frac{d P_n}{d P_1})
$$
with $f = \Lbar(\pi) - \phi$ convex. This completes the proof.

\end{proof}

The converse is also true. The proof follows by similar manipulations and an appeal to the properties of proper losses (see \cite{Reid2009b,Garca-Garca}), and is not inlcuded.

\subsection{Proper Loss Functions}

Here we review material relating to the construction of proper loss functions \citep{Reid2009b,Dawid2007,Parry2012,Grunwald2004}. 

\textbf{Definition} \emph{A loss function $L: \Theta \times \mathcal{P}(\Theta) \rightarrow \RR$ is \emph{proper} if for all $P \in \mathcal{P}(\theta)$
$$
P \in \arginf_{Q \in \mathcal{P}(\Theta)} \EE_{\theta \dist P} L(\theta,Q) 
$$}

Any loss function can be \emph{properized}.

\textbf{Theorem} \emph{Let $L: \Theta \times A \rightarrow \RR$ be a loss. Define 
\begin{align*}
f: \mathcal{P}(\Theta) &\rightarrow A \\
P &\mapsto \arginf_a \EE_{\theta \dist P} L(\theta,a)
\end{align*}
where we arbitrarily pick an $a \in \arginf_a \EE_{\theta \dist P} L(\theta,a)$ if there are multiple. Then $\hat{L}(\theta,P) = L(\theta,f(P))$ is proper.}

It is possible that by using this trick we remove actions $a\in A$. However, for the purpose of calculating Bayes risks we do not require these actions. From $\hat{L}$, one can define a \emph{regret}
$$
D(P,Q) = \EE_{\theta \dist P} \hat{L}(\theta,Q) - \EE_{\theta \dist P} \hat{L}(\theta,P)
$$
which measures how suboptimal the best action is to play against the distribution $Q$ is when played against the distribution $P$. One does not need knowledge of the function $f$ to construct $\hat{L}$, rather one only needs knowledge of the \emph{Bayes risk} 
$$
\Lbar(P) = \inf_a \EE_{\theta \dist P} L(\theta,a).
$$
From this one can reconstruct $\hat{L}$, and hence $L$ for the purposes of calculating Bayes risks. This is achieved by taking the 1-homogenous extension of $\Lbar$ 
\begin{align*}
\tilde{\Lbar}: \RR_+^{|\Theta|} &\rightarrow \RR \\
v &\mapsto \lVert v \rVert_1 \Lbar(\frac{v}{\lVert v \rVert_1})
\end{align*}
and differentiating.

\textbf{Theorem} \emph{For a concave $\Lbar: \mathcal{P}(\Theta) \rightarrow \RR$ 
$$
L(\theta,P) = \langle \delta_\theta , \nabla \tilde{\Lbar}(P) \rangle
$$
is a proper loss.}

The regret from a proper loss is equal to the \emph{Bregman divergence} defined by $\tilde{\Lbar}(P)$ and $\Lbar$,

\textbf{Theorem} \emph{For all concave $\Lbar: \mathcal{P}(\Theta) \rightarrow \RR$ and $P,Q \in \mathcal{P}(\Theta)$
$$
D(P,Q) = D_{\Lbar}(P,Q) = D_{\tilde{\Lbar}}(P,Q) = \tilde{\Lbar}(Q) + \langle \nabla \tilde{\Lbar}(Q) , P - Q \rangle - \tilde{\Lbar}(P).
$$}
All of these properties show that we only need knowledge of $\Lbar$ to compute Bayes risks and values.

\subsection{Supervised Feature Learning}

For a given experiment $T$ and loss function $L$, supervised feature learning methods aim to minimize the feature gap 
$$
\Delta\mathcal{V}_L(\pi,\phi,T)
$$
The following two lemmas \citep{Banerjee2005,Tishby2000} provide means to do this.
\begin{mylemma} The feature gap satisfies
\begin{align*}
\Delta\mathcal{V}_L(\pi,\phi,T) &= \EE_{x \dist \pi_x} \EE_{z\dist \phi(x)} D_{\Lbar}(T^*(x),(\phi \circ T)^*(z)) \\
&= \EE_{x \dist \pi_x} \EE_{z\dist \phi(x)} D_{\Lbar}(P(\Theta|x),P(\Theta|z)) 
\end{align*}
where $D_{\Lbar}$ is the regret induced by $\Lbar$. 
\end{mylemma}

\begin{proof}

For the proof we use the more familiar probability theory notation with $T^*(x) = P(\Theta|x)$, $(\phi \circ T)^* = P(\Theta|z)$, $\pi_X = P(X)$ and so on. One has 
$$
D_{\Lbar}(P,Q) = \Lbar(Q) + \langle \nabla \Lbar(Q) , P - Q \rangle - \Lbar(P), \ P,Q \in P(\Theta).
$$
and
\begin{align*}
\mathcal{V}_L(\pi,T) = \EE_{x\dist P(X)} \Lbar(P(\Theta|x)) \\
\mathcal{V}_L(\pi,\phi \circ T) = \EE_{x\dist P(Z)} \Lbar(P(\Theta|z))
\end{align*}
giving
\begin{align*}
&\EE_{x \dist P(X)} \EE_{z\dist P(Z|x)} D_{\Lbar}(P(\Theta|x),P(\Theta|z)) \\
= &\EE_{x \dist P(X)} \EE_{z\dist P(Z|x)} \Lbar(P(\Theta|z)) + \langle \nabla \Lbar(P(\Theta|z)) , P(\Theta|x) - P(\Theta|z) \rangle - \Lbar(P(\Theta|x)) \\
= &\EE_{z \dist P(Z)} \EE_{x\dist P(X|z)} \Lbar(P(\Theta|z)) + \langle \nabla \Lbar(P(\Theta|z)) , P(\Theta|x) - P(\Theta|z) \rangle - \Lbar(P(\Theta|x)) \\
= & \mathcal{V}_L(\pi,\phi \circ T) - \mathcal{V}_L(\pi,T) + \EE_{z \dist P(Z)} \EE_{x\dist P(X|z)}\langle \nabla \Lbar(P(\Theta|z)) , P(\Theta|x) - P(\Theta|z) \rangle
\end{align*}
Note that $\langle \nabla \Lbar(P(\Theta|z)) , P(\Theta|x) - P(\Theta|z) \rangle$ is affine in $P(\Theta|x)$ and that 

$$
\EE_{x \dist P(X|z)} P(\Theta|x) = P(\Theta|z)
$$
meaning 
$$
\EE_{z \dist P(Z)} \EE_{x\dist P(X|z)}\langle \nabla \Lbar(P(\Theta|z)) , P(\Theta|x) - P(\Theta|z) \rangle = 0
$$
this completes the proof.

\end{proof}

\begin{mylemma}
For $\pi_Z\in \mathcal{P}(Z)$, $\hat{U} \in M(Z,\Theta)$
\begin{align*}
&\inf_\phi \EE_{x \dist \pi_X} \EE_{z\dist \phi(x)} D_{\Lbar}(T^*(x),(\phi \circ T)^*(z)) + \beta I(X,Z)\\
= &\inf_{\phi}\inf_{\hat{\pi}_Z}\inf_{\hat{U}} \EE_{x \dist \pi_X} \EE_{z\dist \phi(x)} D_{\Lbar}(T^*(x),\hat{U}(z)) + \EE_{x \dist \pi_X} D_{KL}(\phi(x) || \hat{\pi}_Z)
\end{align*}
\end{mylemma}

This lemma is proved by the following theorem from \citet{Banerjee2005}.

\begin{mythm}
For any concave $\Lbar: X \rightarrow \RR$ and distribution $P\in \mathcal{P}(X)$ 
$$
\bar{x} \in \arginf_{x} \EE_{y \dist x} D_{\Lbar}(y,x).
$$
\end{mythm}
We can know prove the lemma

\begin{proof}
Once again we use the more standard notation from probability theory. Firstly
$$
I(X;Y) = \EE_{x \dist P(X)} D_{KL}(P(Z|x),P(Z)) = \inf_{\hat{P}(Z)} \EE_{x \dist P(X)} D_{KL}(P(Z|x),\hat{P}(Z)).
$$
since $\EE_{x\dist P(X)} P(Z|x) = P(Z)$. Secondly 
\begin{align*}
\EE_{x \dist P(X)} \EE_{z\dist P(Z|x)} D_{\Lbar}(P(\Theta|x),P(\Theta|z)) &= \EE_{Z \dist P(Z)} \EE_{x\dist P(X|z)} D_{\Lbar}(P(\Theta|x),P(\Theta|z)) \\
&= \EE_{Z \dist P(Z)} \inf_{\hat{P}(\Theta|z)}\EE_{x\dist P(X|z)} D_{\Lbar}(P(\Theta|x),\hat{P}(\Theta|z)) \\
&= \inf_{\hat{P}(\Theta|Z)} \EE_{Z \dist P(Z)} \EE_{x\dist P(X|z)} D_{\Lbar}(P(\Theta|x),\hat{P}(\Theta|z))
\end{align*}
since $\EE_{x\dist P(X|z)} P(\Theta|x) = P(Z|\Theta)$. Combining gives 
\begin{align*}
&\inf_{P(Z|X)} \EE_{x \dist P(X)} \EE_{z\dist P(Z|x)} D_{\Lbar}(P(\Theta|x),P(\Theta|z)) + \beta \EE_{x \dist P(X)} D_{KL}(P(Z|x),P(Z)) \\
= & \inf_{P(Z|X)} \inf_{\hat{P}(\Theta|Z)} \inf_{\hat{P}(Z)} \EE_{x \dist P(X)} \EE_{z\dist P(Z|x)} D_{\Lbar}(P(\Theta|x),\hat{P}(\Theta|z)) + \beta\EE_{x \dist P(X)} D_{KL}(P(Z|x),\hat{P}(Z)).
\end{align*}
This completes the proof.
\end{proof}

These two lemma's yield a family of alternating minimization algorithms that attempt to solve the supervised feature learning problem. The term $\beta I(X;Z)$ can be interpreted as a regularizer that favours $\phi$ that throw away information about $X$.

\subsection{Variational Divergence}

Let $P$ and $Q$ be distributions on a measure space $X$. The variational divergence has the following equivalent forms 
\begin{enumerate}

\item $\|P - Q\| = \sup_{\phi: X \rightarrow [-1,1]} \EE_P \phi - \EE_Q \phi $
\item $\|P - Q\| = \inf_C$ such that $\EE_P \phi - \EE_Q \phi \leq C \lVert \phi \rVert_\infty$ for all bounded $\phi$.
\item $\|P - Q\| = \int_X |P - Q|$ the $l_1$ distance between the probability distributions $P$ and $Q$. 
\item $\|P - Q\| = D_f(P,Q) = \int_X f(\frac{dQ}{dP}) dP$ the f-divergence from $P$ to $Q$ for $f(x) = |x-1|$  
\item $\|P - Q\| = 2 \sup_{A \in \Sigma(X)} P(A) - Q(A), A\subseteq X$

\end{enumerate}
Since $\lVert \ \rVert$ is a $f$-divergence, it also satisfies an information processing theorem \citep{Reid2009b}. For any Markov kernel $T: X \rightsquigarrow Y$,
$$
\lVert P - Q \rVert \geq \lVert T(P) - T(Q) \rVert
$$
Finally if $\pi \in \mathcal{P}(X)$ and $T,U: X \rightsquigarrow Y$ then
$$
\lVert \pi \otimes T - \pi \otimes U \rVert = \EE_{\theta \dist \pi} \lVert T(\theta) - U(\theta) \rVert
$$ 

\subsection{Proof that Weighted Deficiency Satisfies the Triangle Inequality}

We wish to show that
$$
\Delta_\pi (T,U) = \max(\delta_{\pi}(T,U), \delta_{\pi}(U,T))
$$
provides a metric on experiments on $\Theta$ (modulo isomorphism), for priors $\pi$ that do not assign zero mass to some $\theta$. It should be fairly obvious that $\Delta_\pi$ is both symmetric and non negative. All that is left to prove is the triangle inequality. Fix Markov kernels $T_i: \Theta \rightsquigarrow X_i$ as well as Markov kernels $V_1$ and $V_2$ as in the diagram below
{\large 
$$
\xymatrix{
 && \Theta \ar@{~>}[lld]_{T_1} \ar@{~>}[d]^{T_2} \ar@{~>}[rrd]^{T_3} &&\\
X_1 \ar@{~>}[rr]^{V_1} && X_2 \ar@{~>}[rr]^{V_2} && X_3
}
$$
}
We have $\forall \theta \in \Theta$
\begin{align*}
\lVert T_3(\theta) - V_2 \circ V_1 \circ T_1(\theta) \rVert &\leq \lVert T_3(\theta) - V_2 \circ T_2(\theta) \rVert + \lVert V_2\circ T_2(\theta) - V_2 \circ V_1 \circ T_1(\theta) \rVert \\
&\leq \lVert T_3(\theta) - V_2 \circ T_2(\theta) \rVert + \lVert T_2(\theta) - V_1 \circ T_1(\theta) \rVert
\end{align*}
where we have used the fact the variational divergence is a metric and that it satisfies an information processing theorem. Averaging with respect to the prior and taking infimums over $V_1$ and $V_2$ yields
$$
\delta_\pi(T_1||T_3) \leq \delta_\pi(T_1||T_2) + \delta_\pi(T_2||T_3).
$$
Going in the opposite direction and taking maximums yields 
$$
\Delta_\pi(T_1,T_3) \leq \Delta_\pi(T_1,T_2) + \Delta_\pi(T_2,T_3)
$$
the desired result. Even if $\pi$ does give zero mass to some $\theta$, meaning $\Delta_{\pi}$ may not be a metric, the triangle inequality still applies.

\subsection{Proof of the Information Processing Theorem}

\textbf{Theorem (Blackwell-Sherman-Stein)} \emph{$T|U$ if and only if $\mathcal{V}_L(\pi,T) \leq \mathcal{V}_L(\pi,U)$ for all loss functions $L$ and priors $\pi$.}

\begin{proof}

By definition we have 
\begin{align*}
\mathcal{V}_L(\pi,T) &= \inf_{D \in M(\Theta,A)_T} R_L(\pi,D) \\
&= \inf_{D \in M(\Theta,A)_T} f(D)
\end{align*}
Where $f: M(\Theta,A)\rightarrow \RR$. By property $(1)$ of factoring through we have $M(\Theta,A)_{T} \supseteq M (\Theta,A)_{U}$ giving 
\begin{align*}
\mathcal{V}_L(\pi,T) &= \inf_{D \in M(\Theta,A)_T} f(D) \\
&\leq \inf_{D \in M(\Theta,A)_U} f(D) \\
&= \mathcal{V}_L(\pi,U)
\end{align*}

\end{proof}

\subsection{Proof of the Blackwell-Sherman-Stein theorem}

\textbf{Theorem (Blackwell-Sherman-Stein)} \emph{$T|U$ if and only if $\mathcal{V}_L(\pi,T) \leq \mathcal{V}_L(\pi,U)$ for all loss functions $L$ and priors $\pi$.}

Here we prove the converse to the information processing theorem.
\begin{proof}

The condition that $\mathcal{V}_L(\pi,T) \leq \mathcal{V}_L(\pi,U)$ for all loss functions and priors is equivalent to 
$$
\forall L,\ \forall D_U \in M(\Theta,A)_U, \ \exists D_T \in M(\Theta,A)_T \ \text{such that} \ R_L(\delta_\theta,D_Y) \leq R_L(\delta_\theta,d_Y), \ \forall \theta
$$
in words, for any loss and any decision rule based on $U$ there is one based on $T$ that is better. 

To see this note for each $\theta_0 \in \Theta$ we can take loss functions that only care about that particular $\theta_0$, ie $L(\theta,a) = 0$ if $\theta \neq \theta_0$.  Now fix $A$. We have that $R_L(\delta_\theta,-): M(\Theta,A) \rightarrow \RR$ is linear. Furthermore $M(\Theta,A)_{P_{T}}$ and $M(\Theta,A)_U$ are closed convex subsets of $M(\Theta,A)$. As we vary $L$ and $\theta$, $R_L(\delta_\theta,-)$ gives the entire dual space of $M(\Theta,A)$, ie by taking $L(\theta_0,a_0)=1$ and zero for all other $\theta$ and $a$. By the correspondence between risk functions and the dual of $M(\Theta,A)$ means the above is equivalent to
$$
\inf_{D \in M(\Theta,A)_T} \langle \alpha , D \rangle \leq \inf_{D \in M(\Theta,A)_U} \langle \alpha , D \rangle,\ \forall \alpha \in M(\Theta,A)^*
$$
meaning that the support function of $M(\Theta,A)_T$ is less than or equal to the support function of\\ $M(\Theta,A)_T$. Hence 
$$
M(\Theta,A)_U \supseteq M(\Theta,A)_T
$$
Note $A$ was arbitrary. Now let $A = Y$, and note that $U \in M(\Theta,Y)_U$. Hence by the above inclusion, there exists a $U/T: X \rightarrow Y$ with $U = (U/T) \circ T$

\end{proof}

\subsection{Proof of the Randomization theorem}

\textbf{Theorem (Randomization)} \emph{Fix $\epsilon>0$, $T,U$ and $\pi$. $\mathcal{V}_L(\pi,T) \leq \mathcal{V}_L(\pi,U) + \epsilon \lVert L \rVert_{\infty}$ if and only if $\delta_\pi(T,U) \leq \epsilon$ for all loss functions $L$.}

Firstly the if direction. 
\begin{proof}(Forward implication)

If $\delta(T||U) \leq \epsilon$ then there exists a Markov kernel $V$ such that 
$$
\EE_{\theta \dist \pi}\lVert U(\theta) - V \circ T(\theta) \rVert \leq \epsilon
$$
Fix a decision rule $D_U \in M(\Theta,A)_U$. Using $V$ gives a decision rule $D_T \in M(\Theta,A)_T$ by composition, $D_T = (D_U/U) \circ V \circ T$. See the diagram below for a better explanation.
{
\large
$$
\xymatrix{
& \Theta \ar@{~>}[ld]_T \ar@{~>}[rd]^U & &&  \\
X\ar@{~>}[rr]^V && Y \ar@{~>}[rr]^{D_U/U} && A
}
$$
}
By the properties of the variational divergence one has 
\begin{align*}
R_L(\pi,D_T) - R_L(\pi, D_U) &= \EE_{\theta \dist \pi} \EE_{a\dist D_T(\theta)} L(\theta,a) - \EE_{\theta \dist \pi} \EE_{a\dist D_U(\theta)} L(\theta,a) \\
&\leq \EE_{\theta \dist \pi} \lVert D_U(\theta) - D_T(\theta) \rVert \lVert L \rVert_\infty \\
&= \EE_{\theta \dist \pi} \lVert (D_U/U)\circ U(\theta) - ((D_U/U) \circ V \circ T)(\theta) \rVert \lVert L \rVert_\infty \\
&\leq  \EE_{\theta \dist \pi} \lVert U(\theta) -  (V \circ T)(\theta) \rVert \lVert L \rVert_\infty \\
&\leq \epsilon \lVert L \rVert_\infty
\end{align*}
where the first line is the definition of risk, the second follows from one definition of the variational divergence and the third follows from the fact the variational divergence is itself an $f$-divergence and as such satisfies an information processing theorem. Note this holds for all $D_U \in M(\Theta,A)_U$.

Taking infimums over $D_T$ and $D_U$, one has $\mathcal{V}_L(\pi,T) \leq \mathcal{V}_L(\pi,U)  + \epsilon \lVert L \rVert_{\infty}$ for all loss functions.
\end{proof}
We now prove the converse.
\begin{proof}(Converse)

Fix A, a decision rule $D_U\in M(\Theta,A)_U$ and a prior $\pi$ and define a function 
$$
f(L,D_T) = R_L(\pi,D_T) - R_L(\pi,D_U) - \epsilon \lVert L \rVert_{\infty}
$$
that takes a decision rule $D_T\in M(\Theta,A)_T$ and a loss $L$ and returns the difference in Bayes risks of $D_T$ and $D_U$ subtracted by a term that penalizes losses with high magnitude. Note that $f$ is a linear in $D_T$ and concave in $L$. By the conditions in the theorem one has  
$$
\sup_L \inf_{D_T} f(L,D_T) \leq 0.
$$
By the minimax theorem \citep{Komiya1988} , there exists a saddle point $(L^*,D_T^*)$ with 
$$
f(L^*,D_T^*) = \sup_L \inf_{D_T} f(L,D_T) = \inf_{D_T} \sup_L  f(L,D_T).
$$

Hence $f(L,D_T^*) \leq 0, \ \forall L$, meaning 
$$
R_L(\pi,D_T^*) - R_L(\pi,D_U) \leq  \epsilon \lVert L \rVert_{\infty} \ \forall L.
$$ 

This implies that 
$$
\EE_{\theta \dist \pi} \EE_{a\dist D_T^*(\theta)} L(\theta,a) - \EE_{\theta \dist \pi} \EE_{a\dist D_U(\theta)} L(\theta,a) \leq\epsilon \lVert L \rVert_{\infty} \ \forall L
$$

meaning
$$
\lVert D_T^* \otimes \pi - D_U \otimes\pi \rVert = \EE_{\theta \dist \pi} \lVert D_T (\theta) - D_U(\theta) \rVert \leq \epsilon
$$
As $D_U$ was arbitrary, we have for all $D_U\in M(\Theta,A)_U$ there exists a $D_T\in M(\Theta,A)_T$ with $\EE_{\theta \dist \pi} \lVert D_T (\theta) - D_U(\theta) \rVert \leq \epsilon$. 

Once again, take $Y=A$ and note that $U\in M(\Theta,Y)_U$.

\end{proof}

\subsection{Proof of Theorem 1}

\textbf{Theorem} \emph{For all experiments $U,T$ and all priors $\pi$
$$
\Delta_\pi(U,T) = \sup_L \frac{|\mathcal{V}_L(\pi,U) - \mathcal{V}_L(\pi,T)|}{\lVert L \rVert_{\infty}}
$$}

For the proof we require the following very simple lemma.

\begin{mylemma}
For $x,y \in \RR$ if $\forall \epsilon \in \RR$ we have $x \leq \epsilon \Leftrightarrow y \leq \epsilon$ then $x = y$.
\end{mylemma}

\begin{proof}
Suppose that $x \neq y$ and without loss of generality assume that $x \leq y$. Set $\epsilon = \frac{x + y}{2}$. Then $x \leq \epsilon$ and $y \geq \epsilon$, a contradiction.
\end{proof}

Now the theorem.

\begin{proof}
If $\Delta_\pi(U,T) \leq \epsilon$ then $\delta_\pi(U,T) \leq \epsilon$ and $\delta_\pi(T,U) \leq \epsilon$. By the randomization theorem
$$
| \mathcal{V}_L(\pi,U) - \mathcal{V}_L(\pi,T) | \leq \epsilon \lVert L \rVert_\infty ,\ \forall L
$$
hence
$$
\sup_L \frac{|\mathcal{V}_L(\pi,U) - \mathcal{V}_L(\pi,T)|}{\lVert L \rVert_{\infty}} \leq \epsilon.
$$
Conversely if $\sup_L \frac{|\mathcal{V}_L(\pi,U) - \mathcal{V}_L(\pi,T)|}{\lVert L \rVert_{\infty}} \leq \epsilon$
then for all $L$
$$
\mathcal{V}_L(\pi,T) \leq \mathcal{V}_L(\pi,U) + \epsilon \lVert L \rVert_{\infty}
$$
and 
$$
\mathcal{V}_L(\pi,U) \leq \mathcal{V}_L(\pi,T) + \epsilon \lVert L \rVert_{\infty}
$$
which by the randomization theorem gives $\Delta_\pi(U,T) \leq \epsilon$. Combining these two facts and the lemma completes the proof.
\end{proof}

\subsection{Proof of Theorem 2}

\textbf{Theorem} \emph{For all experiments $T: \Theta \rightsquigarrow X$, for all measure spaces $Z$ and for all feature maps $\phi: X \rightsquigarrow Z$ 
$$
\Delta_\pi(T,\phi \circ T) \leq \Delta_{T(\pi)}(id_X,\phi).  
$$}

To prove this theorem we note that from the randomization theorem we have \citep{Torgersen1991}
$$
\Delta_\pi(T,U) = \sup_L \frac{| \mathcal{V}_L(\pi,T) - \mathcal{V}_L(\pi,U) |}{\lVert L \rVert_{\infty}}.
$$
By standard manipulations from decision theory one has
\begin{align*}
\mathcal{V}_L(\pi,T) &= \inf_{d: X \rightsquigarrow A} \EE_{\theta \dist \pi} \EE_{x \dist T(\theta)} \EE_{a \dist d(x)} L(\theta,a) \\
&= \inf_{d: X \rightsquigarrow A} \EE_{x \dist \pi_X} \EE_{\theta \dist T^*(x)} \EE_{a \dist d(x)} L(\theta,a) \\
&= \EE_{x \dist \pi_X} \inf_{a\in A} \EE_{\theta \dist T^*(x)} L(\theta,a) \\
&= \EE_{x \dist \pi_X} \Lbar(T^*(x))
\end{align*}
giving
$$
\Delta_\pi(T,U) = \sup_L \frac{| \EE_{x\dist T(\pi)}\Lbar(T^*(x)) - \EE_{y\dist U(\pi)}\Lbar(U^*(x)) |}{\lVert L \rVert_{\infty}}.
$$
Due to the correspondence between loss functions and their Bayes risks \citep{Reid2009b,Dawid2007,Grunwald2004,Garca-Garca} (In particular one can recover the loss from its Bayes Risk), and the fact that Bayes risks $\Lbar: \mathcal{P}(\Theta) \rightarrow \RR$ are ``attached directly to $\Theta$'', in feature learning we may wish to move them to objects that are attached to $X$. We are inspired by the diagram below.
{\large $$
\xymatrix{
 & & &  Z \\
\Theta \ar@{~>}[rr]^T \ar@{~>}@/^2pc/[rrru]^{\phi \circ T} \ar@{~>}@/_2pc/[rrrd]_T & & X \ar@{~>}[ru]^\phi \ar@{~>}[dr]_{id_X} &  \\
 & & & X  
}
$$}
This is achieved by the following lemma.
\begin{mylemma}

Let $\mathcal{L}(\Theta)$ denote the set of all Bayes risks on $\Theta$ (proper concave functions). Fix a Markov kernel $T: \Theta \rightsquigarrow X$ and a prior $\pi$ on $\Theta$. Then the Markov kernel $T^*: \Theta \rightsquigarrow \Theta$ provided by Bayes rule gives a function
\begin{align*}
\mathcal{\Lbar}[T]: \mathcal{L}(\Theta) &\rightarrow \mathcal{L}(X) \\
\Lbar \mapsto \Lbar \circ T^*
\end{align*}
furthermore for a chain of Markov kernels
{\large $$
\xymatrix{
\Theta \ar@{~>}[rr]^T \ar@{~>}@/_2pc/[rrrr]_{\phi \circ T = U} && X \ar@{~>}[rr]^\phi && Z}
$$}
we have that $\EE_{z \dist U(\pi)} \Lbar(U^*(y)) = \EE_{z \dist U(\pi)} \mathcal{\Lbar}[T](\Lbar)(\phi^*(y))$

\end{mylemma}

\begin{proof}

It should be fairly obvious that $\mathcal{\Lbar}[T](\Lbar)$ is concave as it is the composition of a linear function and a concave function. Furthermore
\begin{align*}
\EE_{z \dist U(\pi)} \mathcal{\Lbar}[T](\Lbar)(\phi^*(z)) &= \EE_{z \dist U(\pi)} \Lbar (T^* \circ \phi^*(z)) \\
&= \EE_{z \dist U(\pi)} \Lbar (U^*(z))
\end{align*}

\end{proof}
we can now prove theorem 2
\begin{proof}
\setcounter{equation}{0}
\begin{align}
\Delta_\pi(T,\phi \circ T) &= \sup_{\Lbar \in \mathcal{L}(\Theta)} \frac{| \EE_{x\dist T(\pi)}\Lbar(T^*(x)) - \EE_{y\dist \phi \circ T(\pi)}\Lbar(T^*\circ \phi^*(x)) |}{\lVert L \rVert_{\infty}} \\
& = \sup_{\Lbar \in \mathcal{L}(\Theta)} \frac{| \EE_{x \dist T(\pi)} \mathcal{\Lbar}[T](\Lbar)(\delta_x) -\EE_{y \dist \phi \circ T(\pi)} \mathcal{\Lbar}[T](\Lbar)(\phi^*(y)) |}{\lVert \mathcal{\Lbar}[T](L) \rVert_\pi} \\
&\leq \sup_{\Lbar \in \mathcal{L}(X)} \frac{| \EE_{x\dist T(\pi)}\Lbar(\delta_x) - \EE_{y\dist \phi \circ T(\pi)}\Lbar(\phi^*(x)) |}{\lVert L \rVert_{\infty}} \\
&= \Delta_{T(\pi)}(id_X,\phi) \nonumber
\end{align}
\end{proof}

\subsection{Proof of Theorem 3}

\textbf{Theorem} \emph{Fix a prior $\pi_X$ and an $\epsilon >0$.  $\phi:X \rightsquigarrow Z$ constitutes generic features of quality $\epsilon$ for $X$ if and only if for the reconstruction problem $(X,Z,\phi,X,L_{01})$, one can find a decision rule $d: Z \rightsquigarrow X$ with 
$$
R_{L_{01}}(\pi_X,d \circ \phi) \leq \frac{\epsilon}{2}.
$$}

First the if.

\begin{proof}(Forward Implication)
Let $T: \Theta \rightsquigarrow X $ be an experiment, $L$ a loss and $\pi$ a prior on $\Theta$ with $T(\pi) = \pi_X$. For any $D_T \in M(\Theta,A)_T$ consider 
$$
D_{\phi \circ T} = (D_T/T)\circ d \circ \phi \circ T \in M(\Theta,A)_{\phi \circ T}.
$$
We wish to calculate
$$
R_L(\pi,D_{\phi \circ T}) = \EE_{\theta \dist \pi} \EE_{x \dist T(\theta)} \EE_{x' \dist (d\circ \phi)(x)} \EE_{a\dist (D_T/T)(x')}L(\Theta,x).
$$
To calculate this risk we consider two cases. Either $x=x'$ with probability $1-\frac{\epsilon}{2}$, giving risk $R_L(\pi,D_T)$, or $x' \neq x$ with probability $\frac{\epsilon}{2}$ with risk bounded above by $2 \lVert L \rVert_\infty$. This upper bound occurs when the action chosen by $d$ after seeing $x$ is the best for $\theta$ and the action chosen by $d$ after seeing $x'$ is the worst for $\theta$. Combining these two gives 
$$
R_L(\pi,D_{\phi \circ T}) \leq (1-\frac{\epsilon}{2}) R_L(\pi,D_T) + \frac{\epsilon}{2} 2 \lVert L \rVert_\infty \leq R_L(\pi,D_T) + \epsilon \lVert L \rVert_\infty.
$$
Note that $T$ and $L$ are arbitrary. Taking infimums over $D_T$ and $D_{\phi \circ T}$ gives 
$$
\mathcal{V}_L(\pi,\phi \circ T) \leq \mathcal{V}_L(\pi,T) + \epsilon \lVert L \rVert_\infty, \ \forall L, \ \forall T, \ \forall \pi \ st\  T(\pi) = \pi_X.
$$

\end{proof}

Now the converse.

\begin{proof}(Converse)

If
$$
\mathcal{V}_L(\pi,\phi \circ T) \leq \mathcal{V}_L(\pi,T) + \epsilon \lVert L \rVert_\infty, \ \forall L, \ \forall T, \ \forall \pi \ st\  T(\pi) = \pi_X
$$
then
$$
\Delta_\pi(\phi \circ T, T) \leq \epsilon , \ \forall T, \ \forall \pi \ st\  T(\pi) = \pi_X
$$
in particular by taking $T= \operatorname{id}_X$ and $\pi = \pi_X$ we have
$$
\Delta_{\pi_X}(\phi, \operatorname{id}_x) \leq \epsilon.
$$
By the definition of weighted deficiency and as $\operatorname{id}_X | \phi $, $\Delta_{\pi_X}(\phi, \operatorname{id}_x) = \delta_\pi(\phi,\operatorname{id}_X)$. Therefore there exists a markov kernel $d: Z \rightsquigarrow X$ with 
$$
\EE_{x\dist \pi_X} \lVert \delta_x - (d\circ\phi)(x)\rVert = 2 \EE_{x\dist \pi_X} \EE_{x' \dist (d\circ\phi)(x) } \mathbb{1}(x' = x) \leq \epsilon
$$
Hence $R_{L_{01}}(\pi_X,d\circ \phi) \leq \frac{\epsilon}{2}$.

\end{proof}

\small
\bibliographystyle{plain}

\end{document}